\documentclass[12pt]{article}
\usepackage[margin=1in]{geometry}

\usepackage{times}
\usepackage{rotating}

\usepackage{pdfpages}

\usepackage{wrapfig}
\usepackage{morenotations, bm}
\usepackage{amsfonts}
\usepackage{mathtools}
\usepackage{upgreek}
\usepackage{afterpage, multicol, multirow}
\usepackage{subfig}
\usepackage{algorithm}
\usepackage{algorithmic}
\usepackage{morenotations, bm, bbm}
\usepackage{amsfonts}
\usepackage{mathtools}
\usepackage{upgreek}
\usepackage{algorithm}
\usepackage{algorithmic,multirow}
\usepackage{arydshln}

\newcommand{\BBR}{\textsc{DRL}}
\newcommand{\basic}{\textsc{RadoCraft}}

\title{Fast Learning from Distributed Datasets without Entity Matching}

\date{}

\author{
Giorgio Patrini\\
{\normalsize The Australian National University \& Nicta}\\
  \texttt{giorgio.patrini@anu.edu.au}\\
\and
  Richard Nock\\
{\normalsize Nicta \& The Australian National University}\\
  \texttt{richard.nock@nicta.com.au}\\
\and
  Stephen Hardy\\
{\normalsize Nicta}\\
  \texttt{stephen.hardy@nicta.com.au}\\
\and
  Tiberio Caetano\\
{\normalsize Ambiata, The Australian National University \& The University of New South Wales}\\
  \texttt{tiberio.caetano@gmail.com}
}

\begin{document}

\maketitle

\begin{abstract}
Consider the following data fusion scenario: two datasets/peers contain the same real-world entities described using partially shared features, \emph{e.g.} banking and insurance company records of the same customer base. Our goal is to learn a classifier in the cross product space of the two domains, in the hard case in which no shared ID is available --\emph{e.g.} due to anonymization. Traditionally, the problem is approached by first addressing entity matching and subsequently learning the classifier in a standard manner. We present an end-to-end solution which bypasses matching entities, 
based on the recently introduced concept of \emph{Rademacher observations} (rados). Informally, we replace the minimisation of a loss over examples, which requires to solve entity resolution, by the \textit{equivalent} minimisation of a (different) loss over rados. Among others, key properties we show are (i) a potentially huge subset of these rados \textit{does not require} to perform entity matching, and (ii) the algorithm that provably minimizes the rado loss over these rados has time and space complexities \textit{smaller} than the algorithm minimizing the equivalent example loss.
Last, we relax a key assumption of the model, that the data is vertically partitioned among peers --- in this case, we would not even know the \textit{existence} of a solution to entity resolution. In this more general setting, experiments validate the possibility of significantly beating even the \textit{optimal} peer in hindsight. 
\end{abstract}

\section{Introduction}\label{sec:intro}

Learning from massively distributed data collections and multiple information sources has become a pivotal problem, yet it
faces critical challenges, among which is the fact that it relies on reconstructing consistent examples from diverse features distributed between different data handling \textit{peers}. Exhaustive search to solve this problem is simply not scalable, nor communication efficient, and sometimes not even accurate \cite{eatAS,zecbtAS}.

\noindent --- A key technical message of our paper is:
\begin{center}
\textit{Entity resolution can be bypassed to carry out supervised learning almost as accurate as if its \textbf{solution} were known.}
\end{center}

\begin{figure}[t]
\centering
\begin{tabular}{cc}
\hspace{-0.2cm}\includegraphics[width=.50\linewidth]{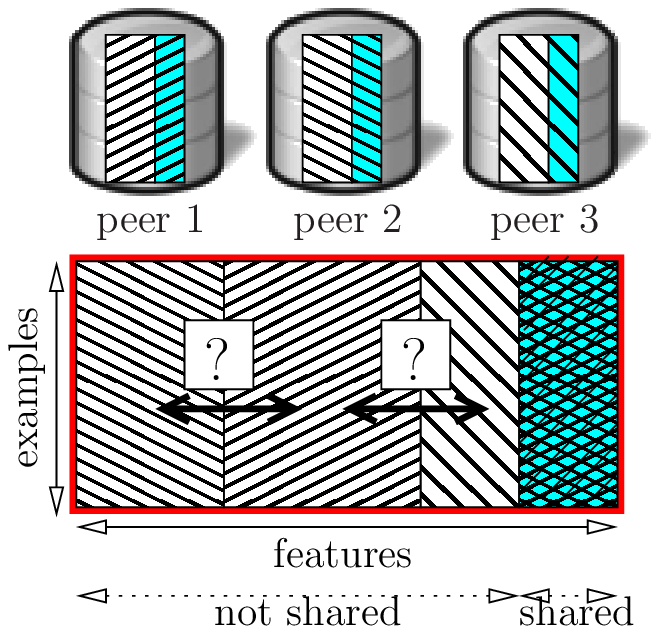} & \hspace{-0.5cm}\includegraphics[width=.50\linewidth]{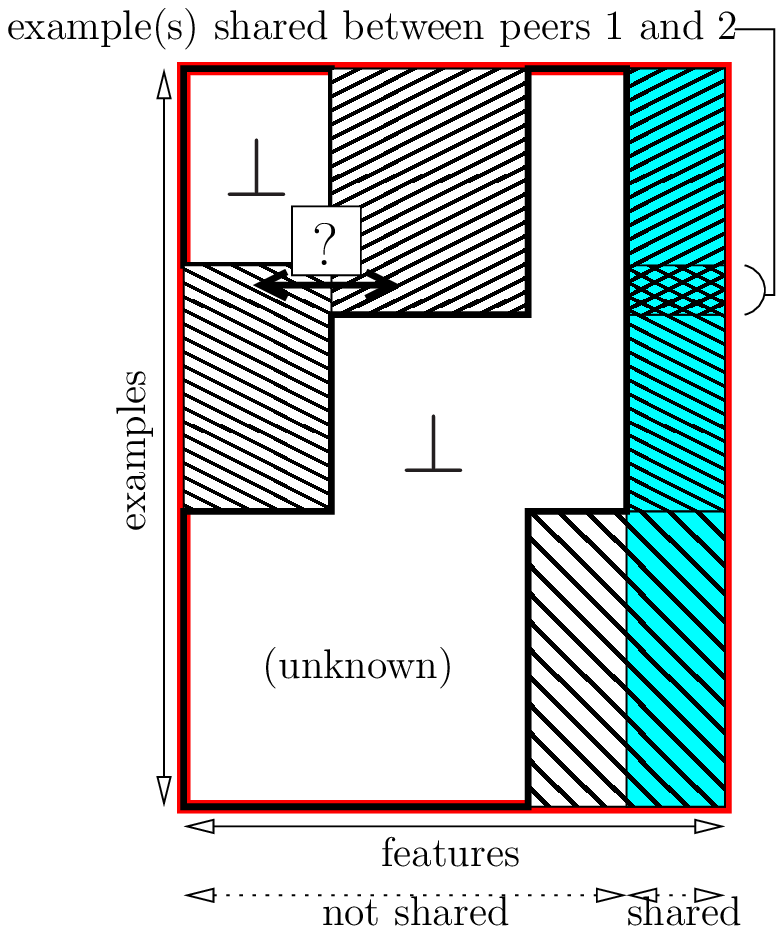} \\
\hspace{-0.2cm} (\textbf{VP}) & \hspace{-0.5cm} (\textbf{G})
\end{tabular}
\caption{Schematic views of our settings, with $p=3$ peers. In both cases, some features (cyan) are described in each peer (best viewed in color) and one these shared features is a class. Non-shared features are split among peers. A so-called \textit{total} sample ${\mathcal{S}}$ is figured by the red rectangle. \textit{Left}: in the vertical partition (\textbf{VP}) case, all peers see different views of the same examples, but do not know who is who among their datasets ("?"). Hence, each bit of the total sample is seen by one peer. \textit{Right}: in the more general setting (\textbf{G}), it is not even known whether one example, viewed by a peer, also exist in other peers' datasets. In this case, there may be a lot of missing data ($\bot$), but it is not known which example has missing data. \label{fig:setting}
}
\end{figure}

A main motivation of this work comes from the reported experience that combining features from different sources leads to better predictive power. For instance, insurance and banking data together can improve fraud detection; shopping records well complement medical history for estimating risk of disease \cite{tedghwTD}; joining heterogeneous data helps prediction in genomics \cite{ldbcjnAS, yvkPN}; security agencies integrate various sources for terrorism intelligence \cite{sPE, cPP, sdklllpwBC}.

Typical data fusion methods however rely on a known map between entities \cite{bnDF}, \textit{i.e.}, peers have partially different views of the \textit{same} examples. Instead, we assume the datasets do not share a common ID, as shown in Figure \ref{fig:setting} ((\textbf{VP}), left); that is, for example, the case when data collection of was performed independently by each peer, or when sources were deliberately anonimized. \emph{Entity resolution} (\textsc{ER}), or entity matching \cite{cDM}, would be the traditional approach for reconciling entities with no shared ID\footnote{This is clearly non trivial: if just two rows in each dataset have the same exact values for the shared features across the $p $ peers, this yields $2^p$ possible matchings for the reconstruction of the two examples involved.}. It approximates a \textsc{join} operation, assuming that some of the attributes are shared, \textit{e.g.}, \textit{age-band}, \textit{gender}, \textit{postcode} (etc.), and hence can be used as ``weak IDs". Most techniques for \textsc{ER} are based on similarity functions and thresholding: candidate entities are selected as matches when their similarity is above a threshold. Both components can be tuned on some ground truth matches and effectively enhanced with learning techniques \cite{bmAD, cDM}. 
The various metrics of ER encompass lots of different parameters, including generality, accuracy, soundness, scalability, parallelizability \cite{rdgLS}. The standard pipeline for learning with ER is depicted in Figure \ref{fig:sett} (left): (1) entities are matched based on similarity and heuristics, (2) they are merged in one unique database and (3) a model is learnt on the joint data. Common issues in fusion, such as \emph{conflicts} and \emph{heterogeneity} \cite{bnDF}, are not considered in this work.

\begin{table}[t]
\centering
\begin{tabular}{cc}
Peer 1 & Peer 2 \\
\begin{tabular}{rc:cc}
 & & \multicolumn{2}{c}{shared}\\ 
& $x_1$ & $x_3$ & $c$\\ \cline{2-4}
$e_1$ &  1 & 1 & 1\\\cline{2-4}
$e_2$ & -1 & 1 & 1\\ \cline{2-4}\cline{2-4}
\end{tabular} & \begin{tabular}{rc:cc}
 & & \multicolumn{2}{c}{shared}\\ 
& $x_2$ & $x_3$ & $c$\\ \cline{2-4}
$e'_1$ &  -1 & 1 & 1\\\cline{2-4}
$e'_2$ & 1 & 1 & 1\\ \cline{2-4}\cline{2-4}
\end{tabular} 
\end{tabular}
\caption{A simple case of the (\textbf{VP}) setting, with $p=2$ peers, with two shared variables $x_3$ and $c$ (the class to predict). This toy example has binary description features and a binary shared feature, but this restriction does not need to hold in the general case. For example, each shared feature can be any categorical/ordinal feature, like ``postcode'', ``age-bracket'', etc.\label{tab:simple}}
\end{table}

\begin{figure}[t]
\centering
\includegraphics[width=.35\linewidth]{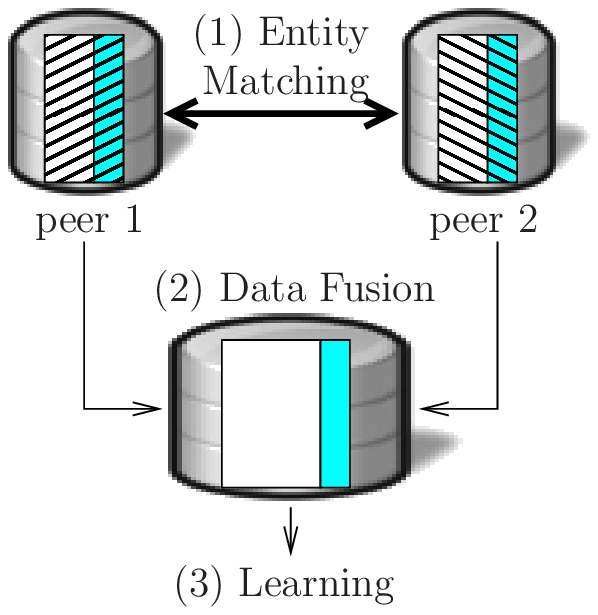} $\:$ \qquad
\includegraphics[width=.35\linewidth]{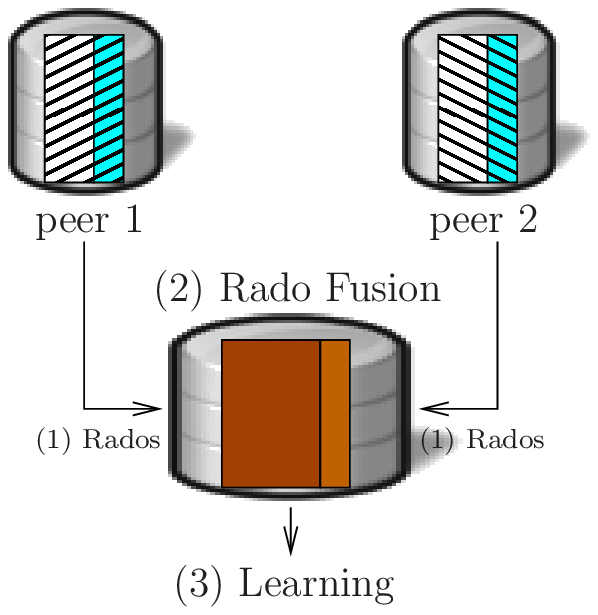} 
\caption{Learning on top of \textsc{ER} (left) or with rados (right). \label{fig:sett}}
\end{figure}
From a high level view, ER integrates data as a pre-process for other tasks. When it comes to learning from ER'ed data, small changes in ER can have large impact on evaluating classifiers, even for simple classifiers as linear models. To see this, suppose we are in the toy example of Table \ref{tab:simple}. Here, all shared variables have the same values, so entity matching has two potential solutions (notice that one of the shared variable is class $c$). One, say ER1, is matching $e_1$ with $e'_1$ and $e_2$ with $e'_2$. We denote the examples obtained by $e_{11} \defeq ((1,-1,1),1)$ and $e_{22}\defeq ((-1,1,1),1)$ (an example is a pair (observation, class)). The other solution, say ER2, is matching $e_1$ with $e'_2$ and $e_2$ with $e'_1$. We denote the examples obtained by $e_{12} \defeq ((1,1,1),1)$ and $e_{21}\defeq ((-1,-1,1),1)$. Now, consider linear classifier $\ve{\theta} = (1,1,1) \in {\mathbb{R}}^3$; the class it gives is the sign of its inner product with an observation, $\ve{\theta}(\ve{z}) \defeq \mathrm{sign}(\ve{\theta}^\top \ve{z})$. While $\ve{\theta}$ classifies perfectly on $\{e_{11}, e_{22}\}$ (zero error), it classifies no better than random on $\{e_{12},e_{21}\}$ (error $50\%$). 

This is a potential consequence of non-accurate ER in a setting in which we \textit{know} that there is a solution to ER, \textit{i.e.} a one-one matching between data peers that recovers the examples as they are in the total sample. What happens if remove this assumption, \textit{i.e.} if we remove the assumption that each example is seen by \textit{all} peers ? This is a much more realistic model. Since there is no shared ID --- and the data may have been anonymized --- we are not even in a situation where we can guarantee that a specific client of the bank \textit{is}, or \textit{is not}, a client of the insurance company. Thus, there may be significant unknown data to reconstruct the total sample ${\mathcal{S}}$ (Figure \ref{fig:setting}), but we do not know which specific examples have missing features. This is our most general setting, (\textbf{G}), shown in Figure \ref{fig:setting} (right).

To cope with (\textbf{VP}) or (\textbf{G}), we use a recently introduced trick to learn from private data \cite{npfRO}: examples are not necessary to learn an accurate linear classifier. We insist on the fact that ``accurate" refers to the quality of the class prediction for observations and examples. The input of the algorithm consists of \textit{Rademacher observations}, rados. One rado is just a sum, over a subset of examples, of the observations times their class. Surprisingly, we can not only learn with data on this form but the output classifier does not require any post-processing since it is the same as if we were learning with example.

\textbf{Contributions} --- Our contribution starts from noticing that many rados are invariant to the selection of different solutions for entity resolution. For example, consider again Table \ref{tab:simple}. Since all classes are positive, computing a rado is just summing observations. Let $\ve{\rado}_{ij,kl}$ be the rado that sums those of examples $e_{ij}$ and $e_{kl}$. Then, surprisingly, regardless of the solution to ER, this rado is the \textit{same}:
\begin{eqnarray}
(\mbox{E1})\:\: \ve{\rado}_{11,22} & = & (1,-1,1) + (-1,1,1) \nonumber \\
 & = & (0,0,2) \nonumber \\
 & = & (1,1,1) + (-1,-1,1) = \ve{\rado}_{12,21} \:\: (\mbox{E2})\nonumber\:\:.
\end{eqnarray}
This, as we show, always holds in the (\textbf{VP}) setting: there exists a huge, \textit{i.e.}, of potential exponential size, set of rados that match the set of rados that could be built \textit{knowing} the true entity resolution. In the most general setting, (\textbf{G}), we show that a very simple transformation of the rados, involving only the shared features, has in expectation the same properties. We give the algorithm that builds these rados. It is easily parallelizable and requires \textit{sublinear} communication, \textit{i.e.} the amount of information that transits is no larger --- and may be much smaller --- than the size of all peers' data.

These "ideal" rados are not just interesting \textit{per se}: learning from them (Figure \ref{fig:sett}, right) is both efficient and accurate. 
We show that using them leads approximating the classifier that would be optimal \textit{on the set of all (ideally ER'ed) examples}. This involves three technical contributions.
The first is an elementary proof that the minimisation of the Ridge regularized square loss \cite{hkRR} (on examples) is equivalent to the minimisation of a regularized rado loss, which we call the $\mathrm{M}$-loss. We then give the closed-form solution for the classifier minimizing the $\mathrm{M}$-loss. Surprisingly, it shows that the minimisation of the regularized $\mathrm{M}$-loss, over the complete (eventually exponential-size) set of "ideal" rados can be done not just in polynomial time: it is in general \textit{faster} than the minimization of the Ridge regularized square loss over examples.
Finally, the optimal $\mathrm{M}$-loss classifier, learnt using only the set of "ideal" rados, converges (as the number of shared features increases) to the minimizer of the Ridge regularized square loss over \textit{all} ideally ER'ed examples. In other words, as the number of shared features increases or as the number of modalities of shared features increases, we are \textit{guaranteed} that the classifier learned over rados will converge to the best classifier learned over examples.

Last, but not least, while we focus on the two-classes setting, description features need not be boolean. There is in fact no restriction apart from the fact that shared features are treated as ordinal instead of plain real: if one feature had as many modalities as there are examples, then there would be no need to address ER.
The rest of this paper is as follows. Section $\S$\ref{sec:prelim} provides preliminaries. $\S$ \ref{sec:bbrados} follows that shows how to learn from distributed data to minimise, indirectly, the Ridge regularized square loss over the ER'ed complete data. $\S$ \ref{sec:exp} presents experimental results. Finally, $\S$ \ref{sec:disc} discusses our approach
 and $\S$ \ref{sec:conc} concludes with open problems.

\section{Preliminaries} \label{sec:prelim}

\paragraph{Learning setting} We let $[n] \defeq \{1,2,\dots,n\}$ for $n \in \mathbb{N_*}$; boldfaces like $\bm{x}$ indicate vectors, whose coordinates are denoted as $x_i$. We briefly recall the task of standard (binary) classification with linear models $\ve{\theta}$ as learning a predictor for label (or class) $y \in\{-1,+1\}$, from a \textit{total} (learning, training) sample ${\mathcal{S}} \defeq \{(\ve{x}_i, y_i), i \in [m]\}$. Each example is an observation-label pair $(\bm{x}_i,y_i) \in \mathcal{X}\times\{-1,+1\}$, with $\mathcal{X} \subseteq \mathbb{R}^d$ the \textit{feature space}, and it is drawn i.i.d. from an unknown distribution. It is convenient to let $\mathcal{X} \defeq \times_{k=1}^d \mathcal{X}_k$. We reserve the word \emph{entity} for a generic record in a dataset, the object of matching, and \textit{attributes} or \textit{features} to its fields. 

\noindent Our learning setting departs from the standard setting in what follows. Instead of one total training sample, we have $p$ (sub)samples, ${\mathcal{S}^j}$ of size $m_j$, $j \in [p]$ for some $p>1$. Each one is defined in its own feature space $\mathcal{X}^j \defeq\times_{k=1}^{d_j} \mathcal X_{j_k}$, where $j_k \in [d], \forall k$. To get a simple case of this framework, shown in Figure \ref{fig:setting}, one may see each ${\mathcal{S}^j}\defeq \{(\ve{x}_i^j, y_i^j), i \in [m_j]\}$ handled by a \textit{peer} $\Peer{j}$.
We rely on the following assumption:
\begin{itemize}
\item [(\textbf{G})] The class, and a subset of features ${\mathcal{J}}$ from $\mathcal{X}$, are shared by all peers. Each other feature is exclusive to one peer.
\end{itemize}
Hence, each of the dimensions of ${\mathcal{J}}$ is in all $\mathcal{X}^j$s. There exists $\dim({\mathcal{J}})+1$ columns that represent the same set of variables among peers, and one of them is the class. This is a very weak and realistic assumption for the features in ${\mathcal{J}}$, as well as for labels, in at least two situations. The first is our setting (\textbf{VP}), which is a gold standard of database frameworks, when the domain is vertically partitioned for the non-shared features, implying $m_j = m_{j'} = m, \forall j, j' \in [p]$. In this case, there exists a one-to-one mapping between the peers' rows, but it may be extremely hard to compute \cite{rdgLS}. The other scenario is when at least one peer has classes, as that turns out to be what is sufficient for all other peers to get labels as well, by the use of algorithms that learn with label proportions \cite{pnrcAN,qsclEL}, as argued in Section \ref{sec:disc}. The assumption that each non-shared feature is seen by exactly one peer simplifies the technicalities: we discuss relaxing this assumption in Section \ref{sec:bbrados} (after Theorem \ref{ide2}).

\paragraph{Rademacher observations} In the standard classification model, a rademacher observation (rado) is a simple transformation of the examples in sample ${\mathcal{S}}$. Let $\bm{\upsigma} \in \Sigma_m \defeq \{-1,1\}^m$. Then rado $\bm{\piup}_{\ve{\upsigma}}$ is $\bm{\piup}_{\ve{\upsigma}} \defeq \sum_{y_i = \upsigma_i} y_i \cdot \bm{x}_i$ \cite{nLG,npfRO}, where $y_i \cdot \bm{x}_i$ is an \textit{edge} vector. In our distributed setting, we extend the definition in the following way. We let $\ve{s} \in {\mathcal{J}}$ denote a \textit{signature}, and $\forall y\in \{-1,+1\}$ and peer $\Peer{j}$, 
\begin{eqnarray}
\bm{\rado}^j_{(\ve{s},y)} & \defeq & \mathrm{proj}_{\mathcal{X}^j\backslash \mathcal{J}}\left(\sum_{i=1}^{m_j} 1_{\mathrm{proj}_{\mathcal{J}}(\bm{x}_i^j) = \ve{s} \wedge y_i^j = y} y_i^j \cdot \bm{x}_i^j\right)  \:\:.\label{defsy}
\end{eqnarray} 
Notation $1_.$ is the indicator function, and $\mathrm{proj}_{\mathcal{I}}(\ve{z})$ denotes the restriction of $\ve{z}$ to ${\mathcal{I}}$. In short, $\bm{\rado}^j_{(\ve{s},y)}$ sums edge vectors local to $\Peer{j}$ whose examples match signature $\ve{s}$ and class $y$. Let ${\mathcal{F}}(\ve{z})$ be the set of features of $\ve{z}$, assumed to be in ${\mathcal{X}}$. We also define, for any ${\mathcal{F}}' \supseteq {\mathcal{F}}(\ve{z})$, $\mathrm{lift}_{\mathcal{F}'}(\ve{z})$ to be the vector $\ve{z}'$ described using $\mathcal{F}'$ such that $\mathrm{proj}_{{\mathcal{F}}(\ve{z})}(\ve{z}') = \ve{z}$ and $\mathrm{proj}_{{\mathcal{F}}'\backslash {\mathcal{F}}(\ve{z})}(\ve{z}') = \ve{0}$. While $\mathrm{proj}_{\mathcal{F}}(\ve{z})$ removes coordinates of $\ve{z}$, $\mathrm{lift}_{\mathcal{F}'}(\ve{z})$ "completes" the coordinates of $\ve{z}$ with zeroes.

By analogy with entity resolution \cite{wmktgER}, we define \textit{block rados} as rados, lifted to ${\mathcal{X}}$, that are the (weighted) sums of examples matching a particular signature and class in all peers.
\begin{definition}
For any $\ve{s} \in {\mathcal{J}}$, $y \in \{-1,1\}$, $u\in {\mathbb{R}}$ the $u$-\textbf{basic block} (BB) rado for pair $(\ve{s},y)$ is 
\begin{eqnarray}
\ve{\rado}^u_{(\ve{s},y)} & \defeq & u \cdot \mathrm{lift}_{\mathcal{X}}(y\cdot \ve{s}) + \sum_{j=1}^{p} \mathrm{lift}_{\mathcal{X}}(\bm{\rado}^j_{(\ve{s}, y)}) \:\:.\label{defbb}
\end{eqnarray} 
\end{definition}
Let ${\mathcal{J}}_+\defeq {\mathcal{J}}\times \{-1,1\}$, and ${\mathcal{J}}_{*} \defeq \{(\ve{s}, y) \in {\mathcal{J}}_+ : \exists j \in [p], \bm{\rado}^j_{(\ve{s}, y)}\neq \ve{0}\}$. This latter set, which can easily be computed from all peers, has cardinal $m_* \defeq |{\mathcal{J}}_{*}|\leq m$, and even $m_* \ll m$ when few features are shared. For any $\ve{u} \in {\mathbb{R}}^{m_*}$, we let ${\mathcal{R}}^{\ve{u}}_{\mbox{\tiny{$\mathrm{B}$}}} \defeq \{\ve{\rado}^{u_i}_{\ve{v}_i}, \forall i \in [m_*]\}$ denote the set of each $u_i$-BB rado, each coordinate of $\ve{u}$ being in one-one correspondence with an element of ${\mathcal{J}}_{*}$ (represented by $\ve{v}_i$). A superset of ${\mathcal{R}}^{\ve{u}}_{\mbox{\tiny{$\mathrm{B}$}}} $ is interesting, that considers all sums of vectors from ${\mathcal{R}}^{\ve{u}}_{\mbox{\tiny{$\mathrm{B}$}}}$:
\begin{eqnarray}
{\mathcal{R}}^{\ve{u}}_{*}  & \defeq & \left\{\sum_{i \in {\mathcal{U}}} \ve{\rado}^{u_i}_{\ve{v}_i}, \forall {\mathcal{U}} \subseteq [m_*]\right\}\:\:.\label{complete}
\end{eqnarray} 
We call ${\mathcal{R}}^{\ve{u}}_{*}$ the set of $\ve{u}$-\textbf{block rados}.  Notice that we may have $|{\mathcal{R}}^{\ve{u}}_{*}| = \Omega(2^{\sum_j |{\mathcal{S}}^j|})$. It is therefore intractable in general to \textit{explicitly} compute ${\mathcal{R}}^{\ve{u}}_{*}$. However, $|{\mathcal{R}}^{\ve{u}}_{\mbox{\tiny{$\mathrm{B}$}}}| = O(\sum_j |{\mathcal{S}}^j|)$ and to compute it, we just need the set of $\bm{\rado}^j_{(\ve{s}, y)}$, hence a communication complexity that can be much smaller than $\sum_j |{\mathcal{S}}^j|$.

\section{Building and learning from BB rados}\label{sec:bbrados}

We address two questions: why/how we can use (basic block) rados to learn accurate classifiers, and how we should fix $\ve{u}$. 

\paragraph{Example vs rado losses} Learning $\ve{\theta}$ on ${\mathcal{S}}$ is done by minimizing a loss function. Here, we consider the Ridge regularized square loss \cite{hkRR} ($\Gamma$ is sym. positive definite, SPD),
\begin{eqnarray}
\ell_{\sql}({\mathcal{S}}, \ve{\theta}; \Gamma) & \defeq & \frac{1}{m} \cdot \sum_i{(1-y_i \ve{\theta}^\top \ve{x}_i)^2} + \ve{\theta}^\top\Gamma \ve{\theta}\:\:.
\end{eqnarray}
It is crucial to remark that this loss is described over the total sample ${\mathcal{S}}$ of examples (see the red rectangle in Figure \ref{fig:setting}). This \textit{is} the loss we want to minimize, exactly or approximately.
One reason we choose this loss is that in the standard classification framework, it admits a simple closed form solution:
\begin{eqnarray}
\ve{\theta}^\star_{\mathrm{ex}} \defeq \arg\min_{\ve{\theta}} \ell_{\sql}({\mathcal{S}}, \ve{\theta}; \Gamma) & = & \left(\X \X^\top+ m\cdot \Gamma\right)^{-1}\ve{\rado}_{\ve{y}}\:\:,\label{defthetaex}
\end{eqnarray}
where $\X \defeq [\ve{x}_1 \vert \ve{x}_2 \vert \cdots \vert \ve{x}_m]$, and so, $\X \X^\top = \sum_i \ve{x}_i \ve{x}^\top_i$. Remark that $\ve{\theta}^\star_{\mathrm{ex}}$ involves one rado, $\ve{\rado}_{\ve{y}}$. For any ${\mathcal{P}} \subseteq \{-1,1\}^m$, we let ${\mathcal{R}}_{{\mathcal{S}}, {\mathcal{P}}}\defeq \{\ve{\rado}_{\ve{\upsigma}} : \ve{\rado}_{\ve{\upsigma}}\in {\mathcal{P}}\}$ denote the set of rados that can be crafted from ${\mathcal{P}}$ using ${\mathcal{S}}$.
\begin{definition}
The $\mathrm{M}$-loss over ${\mathcal{R}}_{{\mathcal{S}}, {\mathcal{P}}}$ of classifier $\ve{\theta}$ is:
\begin{eqnarray}
\ell_{\mathrm{M}}({\mathcal{R}}_{{\mathcal{S}}, {\mathcal{P}}}, \ve{\theta}) & \defeq & -\left(\expect_{{\mathcal{P}}}[\ve{\theta}^\top \ve{\rado}_{\ve{\upsigma}}] - \frac{1}{2}\cdot \variance_{{\mathcal{P}}}[\ve{\theta}^\top \ve{\rado}_{\ve{\upsigma}}] \right)\:\:,\label{defMloss}
\end{eqnarray}
where expectation and variance are computed with respect to the uniform sampling of $\ve{\upsigma}$ in ${\mathcal{P}}$.
\end{definition}
What is inside the parenthesis 
looks like a (vanilla) Markowitz mean-variance criterion \cite{mPS} --- ``vanilla'' because there is no variable coefficient for the risk aversion. What this means is that a good classifier trained on rados should have large ``return'' and small ``risk'', where the risk is the variance of its predictions and the return is its inner product with the expected rado. 

The Theorem to follow shows that what was known for the logistic loss in \cite{npfRO} also holds for the square loss: there exists a loss described over rados, $\ell_{\mathrm{M}}$, such that $\ell_{\sql}(\ve{\theta})$ (dependences on other parameters omitted) is equal to a strictly increasing function of $\ell_{\mathrm{M}}(\ve{\theta})$, \textit{for any} $\ve{\theta}$. Hence, minimizing $\ell_{\sql}(\ve{\theta})$ over examples is \textit{equivalent} to minimizing $\ell_{\mathrm{M}}(\ve{\theta})$ for the \textit{same} classifier.
The proof of the Theorem, elementary, is interesting in itself as it simplifies the long derivation for the equivalence between rado and example losses in \cite{nLG}.
\begin{theorem}\label{equiv1}
Let $\Sigma_m \defeq \{-1,1\}^m$. Then, for any ${\mathcal{S}}$, any $\Gamma$ and \textbf{any} $\ve{\theta}$, $\ell_{\sql}({\mathcal{S}}, \ve{\theta}; \Gamma) = 1 + (4/m) \cdot \ell_{\mathrm{M}}({\mathcal{R}}_{{\mathcal{S}}, \Sigma_m}, \ve{\theta}; \Gamma)$ with
\begin{eqnarray}
\ell_{\mathrm{M}}({\mathcal{R}}_{{\mathcal{S}}, \Sigma_m}, \ve{\theta}; \Gamma) & = & \ell_{\mathrm{M}}({\mathcal{R}}_{{\mathcal{S}}, \Sigma_m}, \ve{\theta}) + \frac{m}{4}   \ve{\theta}^\top\Gamma \ve{\theta} \:\:.\label{defmarkonorm}
\end{eqnarray}
\end{theorem}
\begin{proof}
First, we remark that $\expect_{\Sigma_m}[\ve{\theta}^\top \ve{\rado}_{\ve{\upsigma}}] = \ve{\theta}^\top \expect_{\Sigma_m}[\ve{\rado}_{\ve{\upsigma}}] = (1/2)\cdot \ve{\theta}^\top \ve{\rado}_{\ve{y}}$, since each example participates to half of the $2^m$ rados. Letting $\tilde{v} \defeq 2^{m+2} \cdot \variance_{\Sigma_m}[\ve{\theta}^\top \ve{\rado}_{\ve{\upsigma}}]$, we also have
\begin{eqnarray}
\tilde{v}  & = & 4 \cdot \sum_{\ve{\upsigma} \in \Sigma_m}\left(\ve{\theta}^\top \ve{\rado}_{\ve{\upsigma}} - \frac{1}{2} \cdot \ve{\theta}^\top \ve{\rado}_{\ve{y}}\right)^2\nonumber\\
 & = & \sum_{\ve{\upsigma} \in \Sigma_m}\left(\sum_{i} \upsigma_i \ve{\theta}^\top\ve{x}_i\right)^2\nonumber\\
 & = & \sum_{\ve{\upsigma} \in \Sigma_m}\left[\sum_{i=1}^m (\ve{\theta}^\top\ve{x}_i)^2 + \sum_{i=1}^m\sum_{i' \neq i} \upsigma_i \upsigma_{i'} \ve{\theta}^\top\ve{x}_i \ve{\theta}^\top\ve{x}_{i'}\right]\nonumber\\
 & = & 2^m \cdot \sum_{i=1}^m (\ve{\theta}^\top\ve{x}_i)^2 + \sum_{i=1}^m\sum_{i' \neq i} v_{ii'} \cdot \ve{\theta}^\top\ve{x}_i \ve{\theta}^\top\ve{x}_{i'}\label{pp2}\:\:,
\end{eqnarray}
with $v_{ii'} \defeq \sum_{\ve{\upsigma} \in \Sigma_m}\upsigma_i \upsigma_{i'}$. Now, for any $i\neq i'$, $\upsigma_i \upsigma_{i'}$ takes exactly the same number of times value $+1$ and value $-1$, and so $v_{ii'} = 0, \forall i\neq i'$. We get from eq. (\ref{pp2}) $\variance_{\Sigma_m}[\ve{\theta}^\top \ve{\rado}_{\ve{\upsigma}}] = (1/4) \cdot \sum_{i=1}^m (\ve{\theta}^\top\ve{x}_i)^2 = (1/4) \cdot \sum_{i=1}^m (y_i \ve{\theta}^\top\ve{x}_i)^2$. Finally, 
\begin{eqnarray}
\lefteqn{1 + \frac{4}{m} \cdot \ell_{\mathrm{M}}({\mathcal{S}}, \Sigma_m, \ve{\theta})}\nonumber\\
 & = & 1 - \frac{2}{m} \cdot \sum_{i=1}^m y_i \ve{\theta}^\top \ve{x}_i + \frac{1}{m} \cdot \sum_{i=1}^m (y_i \ve{\theta}^\top\ve{x}_i)^2\nonumber\\
 & = & \frac{1}{m} \cdot \sum_i{(1-y_i \ve{\theta}^\top \ve{x}_i)_2^2}\:\:,
\end{eqnarray}
and we get Theorem \ref{equiv1} by integrating Ridge regularization.
\end{proof}
Hence, minimizing the Ridge regularized square loss over examples is equivalent to minimizing a regularized version of the $\mathrm{M}$-loss, over the complete set of all rados. This set has exponential size. The usual trick would be to randomly subsample this huge set, along with proving good uniform convergence bounds for the $\mathrm{M}$-loss --- this can be done in the same way as for the logistic loss \cite{npfRO}. However, in the case of the square loss, greed pays twice: learning from all rados in ${\mathcal{R}}^{\ve{u}}_{*}$ may be both cheap (computationally) and accurate.\\
\paragraph{Computation and optimality of ${\mathcal{R}}^{\ve{u}}_{*}$} In our distributed context, we do not have access to all rados because we do not assume that we have access to an entity matching function. Yet, we are going to show a first result which is, in a sense, \textit{stronger}: in very general settings, there exists $\ve{u} \in {\mathbb{R}}^{m_*}$ such that ${\mathcal{R}}^{\ve{u}}_{*}$, \textit{systematically} or in expectation, belongs to ${\mathcal{R}}_{{\mathcal{S}}, \Sigma_m}$. This set, ${\mathcal{R}}^{\ve{u}}_{*}$ of potentially exponential size, therefore gives us a set of rados that would have been built \textit{from ${\mathcal{S}}$, had we known the perfect solution to entity matching}. So,  even without carrying out entity matching, we have access to a potentially huge set of "ideal" rados which we can use to learn $\ve{\theta}$ via the minimization of $\ell_{\mathrm{M}}(., \ve{\theta}; \Gamma) $. Furthermore, there exists a simple algorithm to build ${\mathcal{R}}^{\ve{u}}_{\mbox{\tiny{$\mathrm{B}$}}}$.
\begin{figure}[t]
\centering
\includegraphics[width=.65\linewidth]{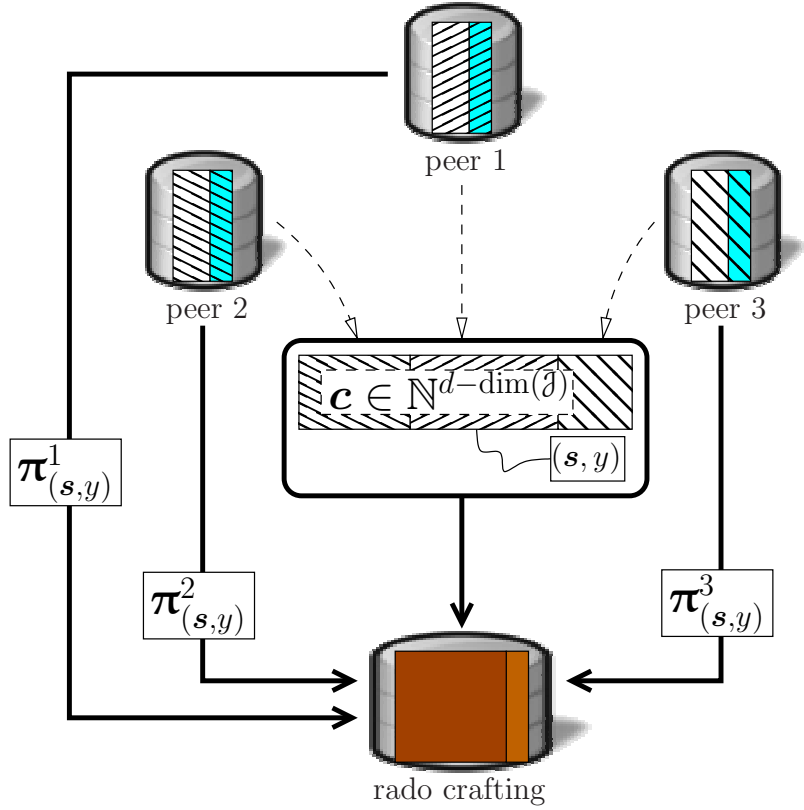} 
\caption{Communication for one BB rado, with $(\ve{s},y)\in {\mathcal{J}}_*$. Counter $\ve{c}$ is defined in Algorithm \basic~(see text).\label{fig:radobuild}}
\end{figure}
The communication protocol, in Figure \ref{fig:radobuild} for $p=3$ peers, summarizes what happens when peers have received message "$\textsc{part}(\ve{s}, y)$" from a "rado crafting" peer implementing Algorithm \ref{algo:BB} below ("$\rightsquigarrow$" symbolizes message sending). Specifically, $\Peer{j}$ does the following:
\begin{itemize}
\item it computes and return $\bm{\rado}^j_{(\ve{s},y)}$; let $C_j$ be the number of examples that are counted in the sum in eq. (\ref{defsy});
\item it updates counter vector $\ve{c}$: for each feature $k \not\in {\mathcal{J}}$ it possesses in its database, it does $c_k \leftarrow c_k + C_j$;
\end{itemize}
Remark that the updates of $\ve{c}$ can easily be done in parallel, as well as the computation of each $\bm{\rado}^j_{(\ve{s},y)}$ for each peer. Letting $\ve{v}_i \defeq (\ve{s}, y) \in {\mathcal{J}}_*$, the corresponding value of $u_i$ is given by:
\begin{eqnarray}
u_i = \tilde{u}_i & \defeq & (\ve{1}^\top \ve{c}) (|d| - \dim({\mathcal{J}}))^{-1} \:\:,
\end{eqnarray}
which is guaranteed to be non-zero since $\ve{v}_i \in {\mathcal{J}}_*$. We now show one of the main results of this paper.
\begin{algorithm}[t]
\caption{\basic($\Peer{1}, \Peer{2}, ..., \Peer{p}$)}
\begin{algorithmic}
\label{algo:BB}
\STATE \textbf{Input} Peers $\Peer{1}, \Peer{2}, ..., \Peer{p}$;
\STATE Step 1: Let ${\mathcal{R}}^{\tilde{\ve{u}}}_{\mbox{\tiny{$\mathrm{B}$}}} \leftarrow \emptyset$;
\STATE Step 2: \textbf{for} $\ve{s} \in {\mathcal{J}}$, $y\in \{-1,+1\}$
\STATE \qquad 2.1: Let $\ve{\rado} \leftarrow \ve{0} \in {\mathbb{R}}^d$, $\ve{c}\leftarrow \ve{0} \in {\mathbb{N}}^{d-\dim({\mathcal{J}})}$;
\STATE \qquad 2.2: \textbf{for} $j\in [p]$
\STATE \qquad\quad 2.2.1: $\ve{\rado} \leftarrow \ve{\rado} + \mathrm{lift}_{\mathcal{X}}(\textsc{part}(\ve{s}, y) \rightsquigarrow \Peer{j})$;
\STATE \qquad 2.3: Let $\tilde{u} \leftarrow (\ve{1}^\top\ve{c})\left(d-\dim({\mathcal{J}})\right)^{-1}$;
\STATE \qquad 2.4: ${\mathcal{R}}^{\tilde{\ve{u}}}_{\mbox{\tiny{$\mathrm{B}$}}}\leftarrow {\mathcal{R}}^{\tilde{\ve{u}}}_{\mbox{\tiny{$\mathrm{B}$}}}\cup ( \tilde{u}\cdot \mathrm{lift}_{\mathcal{X}}(y \cdot \ve{s}) + \ve{\rado})$;
\STATE \textbf{Return} ${\mathcal{R}}^{\tilde{\ve{u}}}_{\mbox{\tiny{$\mathrm{B}$}}}$;
\end{algorithmic}
\end{algorithm}
\begin{theorem}\label{ide1}
In setting (\textbf{VP}), for any $p\geq 2$, \textbf{any} ${\mathcal{S}}$, \textbf{any} ${\mathcal{J}}$, the following holds on the output of Algorithm \ref{algo:BB}: ${\mathcal{R}}^{\tilde{\ve{u}}}_{*}\subseteq {\mathcal{R}}_{{\mathcal{S}}, \Sigma_m}$.
\end{theorem}
\begin{proof} (sketch)
Let $\ve{w}\in {\mathbb{R}}^{m_*}$ be such that $w_i$ is the number of examples in ${\mathcal{S}}$ that match $\ve{v}_i$. The proof follows two steps, (i) ${\mathcal{R}}^{\tilde{\ve{u}}}_{\mbox{\tiny{$\mathrm{B}$}}} = {\mathcal{R}}^{\ve{w}}_{\mbox{\tiny{$\mathrm{B}$}}}$ and (ii) ${\mathcal{R}}^{\ve{w}}_{*}\subseteq {\mathcal{R}}_{{\mathcal{S}}, \Sigma_m}$. Then, (i) is immediate; for (ii), the proof follows once three simple facts are established in the (\textbf{VP}) setting: (a) the true entity matching exists, (b) any $w_i$-BB rado for pair $\ve{v}_i \defeq (\ve{s},y)$ would be obtained as a rado summing the contribution of all examples in ${\mathcal{S}}$ matching the corresponding signature $\ve{s}$ and class $y$, (c) we obtain ${\mathcal{R}}^{\tilde{\ve{u}}}_{\mbox{\tiny{$\mathrm{B}$}}} \subseteq {\mathcal{R}}_{{\mathcal{S}}, \Sigma_m}$, from which follows the Theorem's statement with eq. (\ref{complete}) and the fact that any sum of a subset of rados in ${\mathcal{R}}^{\tilde{\ve{u}}}_{\mbox{\tiny{$\mathrm{B}$}}}$ would also be in ${\mathcal{R}}_{{\mathcal{S}}, \Sigma_m}$ since an example cannot match two distinct couples (signature, class). 
\end{proof}
 Hence, in the (\textbf{VP}) setting, Algorithm \ref{algo:BB} always provides the basis for a (possibly exponential-sized) set ${\mathcal{R}}^{\tilde{\ve{u}}}_{*}$ of these "ideal" rados.
Let us now generalize Theorem \ref{ide1} to (\textbf{G}), which is obviously more difficult to tackle since (i) there may be a huge amount of missing data ($\bot$ in Figure \ref{fig:setting}) and (ii) there would be no one-one correspondence between the peers' examples in general. Yet, there is a interesting property which can be shown in the following (\textbf{R})andomized model: 
each peer's features remain fixed (and all peer's features comply with (\textbf{G})), but there exists a fixed $\ve{\upeta} \in [0,1]^m$ such that example $i$ has probability $\upeta_i$ to be seen by a peer. Let $\overline{{\mathcal{S}}}$ denote the "expected" sample, where each example is weighted by its probability. For any signature $\ve{s}$ and class $y$, $\expect [\ve{\rado}_{(\ve{s}, y)}]$ denotes the expected rado put in ${\mathcal{R}}^{\tilde{\ve{u}}}_{\mbox{\tiny{$\mathrm{B}$}}}$ in step 2.4 of Algorithm \ref{algo:BB}.
\begin{theorem}\label{ide2}
Under (\textbf{R}), $\forall (\ve{s}, y) \in {\mathcal{J}}_+$, $\expect[\ve{\rado}_{(\ve{s}, y)}] \in {\mathcal{R}}_{\overline{{\mathcal{S}}}, \Sigma_m}$.
\end{theorem}
(Proof omitted) Hence, under setting (\textbf{G}), if examples are "seen" independently at random by peers, the expected output of Algorithm \ref{algo:BB} still meets the guarantees of Theorem \ref{ide1} with respect to the expected sample. The fact that ${\mathcal{R}}^{\tilde{\ve{u}}}_{\mbox{\tiny{$\mathrm{B}$}}} \subseteq {\mathcal{R}}_{{\mathcal{S}}, \Sigma_m}$ from Theorem \ref{ide1} is also a consequence of Theorem \ref{ide2} for $\ve{\upeta} = \ve{1}$.
Finally, there is one way to relax further assumption (\textbf{G}), which is to let each feature not shared by all peers to be shared by any subset of peers. In this case, it is possible to modify \basic~so that it \textit{still} builds rados that would meet Theorem \ref{ide2}. This mainly requires a careful adjustment of each counter $\ve{c}$.

\paragraph{Learning from \textit{all} rados of ${\mathcal{R}}^{\tilde{\ve{u}}}_{*}$} The questions that remain are how we minimize the regularized $\mathrm{M}$-loss and, more importantly, what subset of rados from ${\mathcal{R}}^{\tilde{\ve{u}}}_{*}$ we shall use. As already discussed, we choose "greediness" against randomization \cite{npfRO}: instead of picking a (small) random subset of ${\mathcal{R}}^{\tilde{\ve{u}}}_{*}$, we want to use them \textit{all} because we know that all of them are "ideal" or close to being so via Theorems \ref{ide1}, \ref{ide2}. Recall that $|{\mathcal{R}}^{\tilde{\ve{u}}}_{*}|$ may be of exponential size (in $m$, $d$, $|{\mathcal{J}}_*|$, etc.). We now show that if we consider all of ${\mathcal{R}}^{\tilde{\ve{u}}}_{*}$, the optimal $\ve{\theta}^\star_{\mathrm{rad}}$ of $\ell_{\mathrm{M}}({\mathcal{R}}^{\tilde{\ve{u}}}_{*}, \ve{\theta}; \Gamma)$ has an analytic expression which depends \textit{only} on the rados of ${\mathcal{R}}^{\tilde{\ve{u}}}_{\mbox{\tiny{$\mathrm{B}$}}}$. In short, it is even \textit{faster} to compute than $\ve{\theta}^\star_{\mathrm{ex}}$ from ${\mathcal{S}}$ in eq. (\ref{defthetaex}), and can be directly computed from the
output of Algorithm \ref{algo:BB}.
\begin{theorem}\label{thoptsol}
Let $\ve{\theta}^\star_{\mathrm{rad}} \defeq \arg\min_{\ve{\theta}} \ell_{\mathrm{M}}({\mathcal{R}}^{\tilde{\ve{u}}}_{*}, \ve{\theta}; \Gamma)$ (eq. (\ref{defmarkonorm})). Then
\begin{eqnarray}
\ve{\theta}^\star_{\mathrm{rad}} & = & \left(\B\B^\top + \dim_{\mathrm{c}}(\B)\cdot \Gamma\right)^{-1}\B \ve{1}\:\:,\label{optrado}
\end{eqnarray}
where $\B$ stacks in columns the rados of ${\mathcal{R}}^{\tilde{\ve{u}}}_{\mbox{\tiny{$\mathrm{B}$}}}$, and $\dim_{\mathrm{c}}(\B)$ is the number of columns of $\B$.
\end{theorem}
\begin{proof}
The proof uses the following trick: consider any sample ${\mathcal{S}}'$ such that its edge vectors match the basic block rados. Remark that $\X \X^\top = \sum_i (y_i \ve{x}_i) (y_i \ve{x}_i)^\top$ in eq. (\ref{defthetaex}) depends only on edge vectors, and so, since $\ve{\rado}_{\ve{y}} = \B\ve{1}$, the optimal square loss classifier on ${\mathcal{S}}'$ is $\ve{\theta}^\star_{\mathrm{rad}}$ in eq. (\ref{optrado}), which, through Theorem \ref{equiv1}, is also the optimal classifier on $\ell_{\mathrm{M}}({\mathcal{R}}^{\tilde{\ve{u}}}_{*}, \ve{\theta}; \Gamma)$.
\end{proof}
When $m_* = m$, each element of ${\mathcal{R}}^{\tilde{\ve{u}}}_{\mbox{\tiny{$\mathrm{B}$}}}$ is in fact an example, and we retrieve eq. (\ref{defthetaex}). One consequence of Theorem \ref{thoptsol} is the following convergence property which we sketch: in the (\textbf{VP}) setting, for any $\varepsilon \geq 0$, there exists a minimal size for ${\mathcal{J}}_*$ such that $\ve{\theta}^\star_{\mathrm{rad}}$ will be $\varepsilon$-close to $\ve{\theta}^\star_{\mathrm{ex}}$, where the closeness can be measured by $\|\ve{\theta}^\star_{\mathrm{rad}} - \ve{\theta}^\star_{\mathrm{ex}}\|_2$ or $|\mathrm{cos}(\ve{\theta}^\star_{\mathrm{rad}}, \ve{\theta}^\star_{\mathrm{ex}})|$.
The statement of \BBR~(Distributed Rado-Learn) is given in Algorithm \ref{algo:BBR}.  In Step 1, "$\mathrm{column}(.)$" takes a set of vectors and put them in column in a matrix.
\begin{algorithm}[t]
\caption{\BBR($\Peer{1}, \Peer{2}, ..., \Peer{p}; \Gamma$)}
\begin{algorithmic}
\label{algo:BBR}
\STATE \textbf{Input} Peers $\Peer{1}, \Peer{2}, ..., \Peer{p}$, SPD matrix $\Gamma$, $\upgamma > 0$;
\STATE Step 1: $\B\leftarrow \mathrm{Column}(\basic(\Peer{1}, \Peer{2}, ..., \Peer{p}))$;
\STATE Step 2: $\bm{\theta} \leftarrow \left(\B\B^\top + \upgamma \cdot \Gamma\right)^{-1}\B \ve{1}$;
\STATE \textbf{Return} $\bm{\theta}$;
\end{algorithmic}
\end{algorithm}

\section{Experiments}\label{sec:exp}

\begin{table}[t]
\begin{center}
\begin{tabular}{l|rr|c|ll||l}\hline\hline
Domain & $m$ & $d$ & $\min_j\hat{p}_{\mathrm{err}}(\Peer{j})$ & \multicolumn{1}{c}{$p$} & \multicolumn{1}{c}{$\dim({\mathcal{J}})$} & Results\\ \hline 
Wine & 178 & 12 & 0.07 & $\{2, 3, ..., 8\}$& $\{1, 2, 3, 4\}$ & Table \ref{tab:exp_wine}\\
Sonar & 208 & 60 & 0.29 & $\{2, 3, ..., 16\}$& $\{1, 2, ..., 20\}$ &  Table \ref{tab:exp_sonar}\\
Ionosphere & 351 & 33 & 0.20 & $\{2, 3, ..., 9\}$ & $\{1, 2, ..., 9\}$& Table \ref{tab:exp_ionosphere}\\
Mice & 1 080 & 77 & 0.30 & $\{2, 3, ..., 20\}$& $\{1, 2, ..., 20\}$& Table  \ref{tab:exp_mice}\\
Winered & 1 599 & 11 & 0.26 & $\{2, 3, ..., 7\}$& $\{1, 2, 3, 4\}$ & Table \ref{tab:exp_winered}\\
Steelplates & 1 941 & 33 & 0.16& $\{2, 3, ..., 14\}$& $\{1, 2, ..., 5\}$  &Table \ref{tab:exp_steelplates}\\
Statlog & 4 435 & 36 & 0.05&$\{2, 3, ..., 30\}$& $\{1, 2, ..., 5\}$ & Table \ref{tab:exp_statlog}\\
Winewhite & 4 898 & 11 & 0.32&$\{2, 3, ..., 7\}$& $\{1, 2, 3, 4\}$&  Table \ref{tab:exp_winewhite}\\
Page & 5 473 & 10 & 0.21 & $\{2, 3, ..., 6\}$& $\{1, 2, 3, 4\}$ & Table \ref{tab:exp_page}\\
Firmteacher & 10 800 & 16 & 0.26 & $\{2, 3, ..., 7\}$& $\{1, 2, ..., 7\}$ & Table \ref{tab:exp_firmteacher}\\ 
Phishing & 11 055 & 30 & 0.11 & $\{2, 3, 4, 5\}$& $\{1, 2, 3, 4\}$ & Table \ref{tab:exp_phishing}\\ \hline\hline
\end{tabular}
\end{center}
\caption{UCI domains used in our experiments \cite{blUM}, with for each the indication of the total number of features ($d$), examples ($m$) and the error of the optimal peer in hindsight obtained in our experiments, $\min_j\hat{p}_{\mathrm{err}}(\Peer{j})$. Two of the right columns present, for each domain, the range of values for the number of peers ($p$) and the number of shared features ($\dim({\mathcal{J}})$) considered. Experiments were performed considering \textit{all} possible combinations of values of $p$ and $\dim({\mathcal{J}})$ within the allocated sets. The rightmost column points to the Table collecting specific results for each domain. \label{t-summary}}
\end{table}

\paragraph{Algorithms} We have evaluated the leverage that \BBR~provides compared to the peers, that would learn using only their local dataset. Each peer $\Peer{j}$ estimates learns through a ten-folds stratified cross-validation (CV) minimization of $\ell_{\sql}({\mathcal{S}}^j, \ve{\theta}; \upgamma\cdot\mathrm{Id}_{d_j})$ (see eq. (\ref{defthetaex})), where $\upgamma$ is also locally optimized through a ten-folds CV in set ${\mathcal{G}} \defeq \{.01, 1.0, 100.0\}$. \BBR~minimizes $\ell_{\mathrm{M}}({\mathcal{R}}^{\tilde{\ve{u}}}_{*}, \ve{\theta}; \Gamma)$ (solution in eq. (\ref{optrado})) where ${\mathcal{R}}^{\tilde{\ve{u}}}_{\mbox{\tiny{$\mathrm{B}$}}}$ is built using \basic, with the set of all peers as input.

We have carried out a very simple optimisation of the regularisation matrix of \BBR~as a diagonal matrix which weights differently the shared features, $\Gamma \defeq \mathrm{Diag}(\mathrm{lift}_{\mathcal{X}}(\mathrm{proj}_{\mathcal{J}}(\ve{1}))) + \upgamma\cdot \mathrm{Diag}(\mathrm{lift}_{\mathcal{X}}(\mathrm{proj}_{\mathcal{X}\backslash \mathcal{J}}(\ve{1})))$, for $\upgamma \in {\mathcal{G}}$. $\upgamma$ is optimized by a 10-folds CV on ${\mathcal{I}}_*$. CV is performed on \textit{rados} as follows: first, ${\mathcal{R}}^{\tilde{\ve{u}}}_{\mbox{\tiny{$\mathrm{B}$}}}$ is split in 10 folds, ${\mathcal{R}}^{\tilde{\ve{u}}}_{\mbox{\tiny{$\mathrm{B}$}}, \ell}$, for $\ell = 1, 2, ..., 10$. Then, we repeat for $\ell = 1, 2, ..., 10$ (and then average) the following CV routine:
\begin{enumerate}
\item \BBR~is trained using ${\mathcal{R}}^{\tilde{\ve{u}}}_{\mbox{\tiny{$\mathrm{B}$}}}\backslash {\mathcal{R}}^{\tilde{\ve{u}}}_{\mbox{\tiny{$\mathrm{B}$}}, \ell}$;
\item \BBR's solution, $\ve{\theta}^\star_{\mathrm{rad}}$, is evaluated on ``test rados'' by computing $\ell_{\mathrm{M}}({\mathcal{R}}^{\tilde{\ve{u}}}_{\mbox{\tiny{$\mathrm{B}$}}, \ell}, \ve{\theta}^\star_{\mathrm{rad}}; \Gamma)$.
\end{enumerate}
The expression of $\Gamma$ for rados exploits the idea that the estimations related to a shared feature can be much more accurate than for another, non shared feature.\\

\paragraph{Domain generation} we have used a dozen UCI domains \cite{blUM}, presented in Table \ref{t-summary}. For each domain, we have varied (i) the number of peers $p$, (ii) the number of shared features $\dim({\mathcal{J}})$, and (iii) the number $b$ of numeric modalities ("bins") each shared feature was reduced to (controls the size of ${\mathcal{I}}_*$). The training sample is split among peers, each keeping record of ${\mathcal{I}}$ and its own features (non shared features are evenly partitioned among peers). 
Finally, for some $p_s \in [0,1]$, each peer $\Peer{j}$ selects a proportion $p_s$ of its examples index and for each of them, another peer $\Peer{j'}$, chosen at random, gets the example as well (on its own set of features ${\mathcal{X}}^{j'}$). When $p_s = 0$, this is setting (\textbf{VP}).
We then run \textit{all} algorithms for \textit{each} value $p, \dim({\mathcal{J}}), b, p_s$. As we shall see, $b$ appears to have a relatively small influence compared to the other factors, so we mainly report results combining various values for $p, \dim({\mathcal{J}})$ and $p_s$, for the range of values of $p, \dim({\mathcal{J}})$ specified in Table \ref{t-summary}, and for $p_s \in \{0.0, 0.2\}$. We have chosen $b=4$ for all domains, except when it is not possible (if for example all features are boolean), in which case we pick $b=2$. Table \ref{t-summary} also provides the smallest test error obtained for a peer among all runs for each domain: this is an indication of the room of improvement for \BBR, and it also shows that in general, at least some (and in fact most) peers were always very significantly better than random guessing, a safe-check that \BBR~is not just beating unbiased coins. 

\paragraph{Metric} We used two metrics. The first, 
\begin{eqnarray}
\Delta & \defeq & \hat{p}_{\mathrm{err}}(\textsc{\BBR}) - \min_j\hat{p}_{\mathrm{err}}(\Peer{j}) \:\: (\in [-1,1]) \:\:,\label{defdelta}
\end{eqnarray}
is the test error for \BBR~minus that of the \textit{optimal} peer \textit{in hindsight} (since we consider the peer's test error). when $\Delta < 0$, \BBR~beats \textit{all} peers. For example, Table \ref{tab:exp_page} (left) provides the results obtained on UCI domain page. We see that for almost all combinations of $p$ and $\dim({\mathcal{J}})$, \BBR~beats all peers.

To evaluate the statistical significance, we compute 
\begin{eqnarray}
\mathrm{q} & \defeq & \mbox{ proportion of peers \textit{statistically} beaten by }\BBR \:\: (\in [0,1])\:\:.
\end{eqnarray}
To compute the test, we use the powerful Benjamini-Hochberg procedure on top of paired $t$-tests with $q^* = p{\mbox{-val}} = 0.05$, \cite{bhCT}; $q=0.8$ surface helps see when \BBR~\textit{statistically beats all peers}. For example, Table \ref{tab:exp_page} (right) displays that \BBR~does not always \textit{statistically} beat all peers when $\Delta < 0$, yet it manages to stastically beat all of them in approximately one third to one half of the total tests, which implies that, on this domain, there is a significant chance that \BBR~improves on the peers, regardless of their number and the number of shared features.\\
\paragraph{Results} Due to the large number of domains considered, results are split among different tables, one for each domain in general, in Tables \ref{tab:exp_page}, \ref{tab:exp_mice}, \ref{tab:exp_sonar}, \ref{tab:exp_firmteacher}, \ref{tab:exp_ionosphere}, \ref{tab:exp_winered}, \ref{tab:exp_winewhite}, \ref{tab:exp_phishing}, \ref{tab:exp_wine}, \ref{tab:exp_statlog}, \ref{tab:exp_steelplates} (also referenced in Table \ref{t-summary}). All domains display that 
 there exists regimes ($p, \dim({\mathcal{J}})$) for which \BBR~improves on all peers, in some cases significantly. Sometimes, the improvement is sparse (phishing), but sometimes it is quite spectacular and in fact (almost) systematic (page, ionosphere, steelplates). domain steelplate's case is interesting, since the so-called \textit{Oracle}, \textit{i.e.} the learner that leans from the complete training fold \textit{before} it is split among peers --- and therefore knows the solution to entity matching ---, has for this domain almost optimal error, but local peers are in fact very far from this optimum. This indicates that many features, properly combined, are necessary to attain the best performances. \BBR's performances are close to the Oracle, which accounts for the huge gap in classification compared to peers --- sometimes, \BBR's test error is smaller than that of the \textit{best} peer by more than 20$\%$ ---, and so it seems that \BBR~indeed successfully bypasses entity matching to learn a classifier that almost matches the Oracle's performances, and therefore represents a very significant leverage of each peer's data.

To drill down into more specific results, Table \ref{tab:exp_is} (left) displays that binning indeed does not affect significantly \BBR~on average, which is also good news, since it means that there is no restriction on the shared features for \BBR~to perform well: shared features can be binary, or categorical with any number of modalities. Table \ref{tab:exp_fm} displays that while the CV tuning of $\Gamma$ offers leverage to \BBR~(\textit{vs} $\Gamma = \mathrm{Id}_d$) in general (firmteacher), there are some (rare) domains (mice) on which relying on the simplest $\Gamma = \mathrm{Id}_d$ improves upon the results of CV. This, we believe, comes from the fact that CV as we have carried out is certainly not optimal because one rado can aggregate any number of examples. Last, Table \ref{tab:exp_is} (right) drills down a bit more into the performances of \BBR~with respect to those of the Oracle on a domain for which \BBR~obtains somehow ``median'' performances among all domains, sonar. The Oracle (10-folds CV from the \textit{total} ER'ed ${\mathcal{S}}$) is \textit{idealistic} since in general we do not know the solution to ER, yet it gives clues on how close \BBR~may be from the "graal". Interestingly, \BBR~comes frequently under the statistical significance radar ($\alpha = 0.05$). In notable cases (more frequent as $p_s$ increases), \BBR~beats Oracle --- but not significantly. Aside from theory, these are good news 
 as \BBR~does not assume ER'ed data, and uses an amount of data which can be $\sim p^2$ times \textit{smaller} than Oracle.

\begin{table}[t]
\centering
\begin{tabular}{ccc}\hline\hline
\rotatebox[origin=c]{90}{$p_s = 0.0$ \hspace{-4cm}} & \includegraphics[trim=10bp 30bp 40bp 30bp,clip,width=.45\linewidth]{{{Plots/page/results_12th__14h_48m_10s_cvex_true_cvrados_true_normalize_shared_features_true_method_reg_SIMPLE_WEIGHT_P_EX_0.00_opt_local_vs_reg}}} & \includegraphics[trim=10bp 30bp 40bp 30bp,clip,width=.45\linewidth]{{{Plots/page/results_12th__14h_48m_10s_cvex_true_cvrados_true_normalize_shared_features_true_method_reg_SIMPLE_WEIGHT_P_EX_0.00_prop_BH_loc_beaten_by_reg}}}\\
\rotatebox[origin=c]{90}{$p_s = 0.2$ \hspace{-4cm}} & \includegraphics[trim=10bp 30bp 40bp 30bp,clip,width=.45\linewidth]{{{Plots/page/results_12th__14h_48m_10s_cvex_true_cvrados_true_normalize_shared_features_true_method_reg_SIMPLE_WEIGHT_P_EX_0.20_opt_local_vs_reg}}} & \includegraphics[trim=10bp 30bp 40bp 30bp,clip,width=.45\linewidth]{{{Plots/page/results_12th__14h_48m_10s_cvex_true_cvrados_true_normalize_shared_features_true_method_reg_SIMPLE_WEIGHT_P_EX_0.20_prop_BH_loc_beaten_by_reg}}}\\ \hline\hline
\end{tabular}
\caption{Results on domain page: plots of $\Delta \defeq \hat{p}_{\mathrm{err}}(\textsc{\BBR}) - \min_j\hat{p}_{\mathrm{err}}(\Peer{j})$ (left) and $q=$ prop. peers \textit{simultaneously} beaten by \BBR~(right) as a function of the number of peers $p$ and the number of shared features $\dim({\mathcal{J}})$. Top: proportion of shared examples $p_s = 0.0$ (setting (\textbf{VP})); bottom: proportion of shared examples $p_s = 0.2$. The isoline on the left plots is $\Delta = 0$.\label{tab:exp_page}}
\end{table}

\begin{table}[t]
\centering
\begin{tabular}{ccc}\hline\hline
\rotatebox[origin=c]{90}{$p_s = 0.0$ \hspace{-4cm}} & \includegraphics[trim=10bp 30bp 40bp 30bp,clip,width=.45\linewidth]{{{Plots/mice/results_23th__20h_0m_8s_cvex_true_cvrados_true_normalize_shared_features_true_method_reg_SIMPLE_WEIGHT_P_EX_0.00_opt_local_vs_reg}}} & \includegraphics[trim=10bp 30bp 40bp 30bp,clip,width=.45\linewidth]{{{Plots/mice/results_29th__9h_45m_19s_cvex_true_cvrados_true_normalize_shared_features_true_method_reg_SIMPLE_WEIGHT_P_EX_0.00_prop_BH_loc_beaten_by_reg}}}\\
\rotatebox[origin=c]{90}{$p_s = 0.2$ \hspace{-4cm}} & \includegraphics[trim=10bp 30bp 40bp 30bp,clip,width=.45\linewidth]{{{Plots/mice/results_23th__20h_0m_8s_cvex_true_cvrados_true_normalize_shared_features_true_method_reg_SIMPLE_WEIGHT_P_EX_0.20_opt_local_vs_reg}}} & \includegraphics[trim=10bp 30bp 40bp 30bp,clip,width=.45\linewidth]{{{Plots/mice/results_29th__9h_45m_19s_cvex_true_cvrados_true_normalize_shared_features_true_method_reg_SIMPLE_WEIGHT_P_EX_0.20_prop_BH_loc_beaten_by_reg}}}\\ \hline\hline
\end{tabular}
\caption{Results on domain mice, using the same convention as Table \ref{tab:exp_page}.\label{tab:exp_mice}}
\end{table}

\section{Discussion and related work}\label{sec:disc}

\begin{table}[t]
\centering
    \begin{tabular}{|l|c|c|} 
    \hline
     \emph{Metric} & \textsc{ER} + \emph{Learning} &  \textsc{RadoCraft} + \textsc{DRL} \\ \hline
     	Assumption: shared IDs & no & no \\
    	Assumption: some shared variables & necessary & necessary \\
	Assumption: shared labels & no & may be relaxed \\
	Fusion / Rados crafting & $O(m^2 / m^\star \cdot T_{sim})$ & $O(m)$\\
	Communication  & $m \times d$ & $m^\star \times d, ~m^\star \ll m$\\
	Learning problem & $m \times d$ & $m^\star \times d, ~m^\star \ll m$\\
	Privacy & complex & many guarantees \\
    \hline
    \end{tabular}
        \caption{Multiple metrics of comparison between learning on top of \textsc{ER} and our approach. Time complexity are estimated for 2 peers in the (\textbf{VP}) scenario, assuming all blocks of equal size. See Section \ref{sec:disc} for details.}
        \label{table:2}
\end{table}

\begin{table}[t]
\centering
\begin{tabular}{ccc}\hline\hline
\rotatebox[origin=c]{90}{$p_s = 0.0$ \hspace{-4cm}} & \includegraphics[trim=10bp 30bp 40bp 30bp,clip,width=.45\linewidth]{{{Plots/sonar/results_22th__18h_46m_55s_cvex_true_cvrados_true_normalize_shared_features_true_method_reg_SIMPLE_WEIGHT_P_EX_0.00_opt_local_vs_reg}}} & \includegraphics[trim=10bp 30bp 40bp 30bp,clip,width=.45\linewidth]{{{Plots/sonar/results_28th__14h_29m_53s_cvex_true_cvrados_true_normalize_shared_features_true_method_reg_SIMPLE_WEIGHT_P_EX_0.00_prop_BH_loc_beaten_by_reg}}}\\
\rotatebox[origin=c]{90}{$p_s = 0.2$ \hspace{-4cm}} & \includegraphics[trim=10bp 30bp 40bp 30bp,clip,width=.45\linewidth]{{{Plots/sonar/results_22th__18h_46m_55s_cvex_true_cvrados_true_normalize_shared_features_true_method_reg_SIMPLE_WEIGHT_P_EX_0.20_opt_local_vs_reg}}} & \includegraphics[trim=10bp 30bp 40bp 30bp,clip,width=.45\linewidth]{{{Plots/sonar/results_28th__14h_29m_53s_cvex_true_cvrados_true_normalize_shared_features_true_method_reg_SIMPLE_WEIGHT_P_EX_0.20_prop_BH_loc_beaten_by_reg}}}\\ \hline\hline
\end{tabular}
\caption{Results on domain sonar, using the same convention as Table \ref{tab:exp_page}.\label{tab:exp_sonar}}
\end{table}

\begin{table}[t]
\centering
\begin{tabular}{ccc}\hline\hline
\rotatebox[origin=c]{90}{$p_s = 0.0$ \hspace{-4cm}} & \includegraphics[trim=10bp 30bp 40bp 30bp,clip,width=.45\linewidth]{{{Plots/firmteacher2/results_24th__0h_17m_28s_cvex_true_cvrados_true_normalize_shared_features_true_method_reg_SIMPLE_WEIGHT_P_EX_0.00_opt_local_vs_reg}}} & \includegraphics[trim=10bp 30bp 40bp 30bp,clip,width=.45\linewidth]{{{Plots/firmteacher2/results_28th__17h_40m_34s_cvex_true_cvrados_true_normalize_shared_features_true_method_reg_SIMPLE_WEIGHT_P_EX_0.00_prop_BH_loc_beaten_by_reg}}}\\
\rotatebox[origin=c]{90}{$p_s = 0.2$ \hspace{-4cm}} & \includegraphics[trim=10bp 30bp 40bp 30bp,clip,width=.45\linewidth]{{{Plots/firmteacher2/results_24th__0h_17m_28s_cvex_true_cvrados_true_normalize_shared_features_true_method_reg_SIMPLE_WEIGHT_P_EX_0.20_opt_local_vs_reg}}} & \includegraphics[trim=10bp 30bp 40bp 30bp,clip,width=.45\linewidth]{{{Plots/firmteacher2/results_28th__17h_40m_34s_cvex_true_cvrados_true_normalize_shared_features_true_method_reg_SIMPLE_WEIGHT_P_EX_0.20_prop_BH_loc_beaten_by_reg}}}\\ \hline\hline
\end{tabular}
\caption{Results on domain firmteacher, using the same convention as Table \ref{tab:exp_page}.\label{tab:exp_firmteacher}}
\end{table}

\begin{table}[t]
\centering
\begin{tabular}{ccc}\hline\hline
\rotatebox[origin=c]{90}{$p_s = 0.0$ \hspace{-4cm}} & \includegraphics[trim=10bp 30bp 40bp 30bp,clip,width=.45\linewidth]{{{Plots/ionosphere_new/results_27th__10h_36m_29s_cvex_true_cvrados_true_normalize_shared_features_true_method_reg_SIMPLE_WEIGHT_P_EX_0.00_opt_local_vs_reg}}} & \includegraphics[trim=10bp 30bp 40bp 30bp,clip,width=.45\linewidth]{{{Plots/ionosphere_new/results_28th__9h_10m_53s_cvex_true_cvrados_true_normalize_shared_features_true_method_reg_SIMPLE_WEIGHT_P_EX_0.00_prop_BH_loc_beaten_by_reg}}}\\
\rotatebox[origin=c]{90}{$p_s = 0.2$ \hspace{-4cm}} & \includegraphics[trim=10bp 30bp 40bp 30bp,clip,width=.45\linewidth]{{{Plots/ionosphere_new/results_27th__10h_36m_29s_cvex_true_cvrados_true_normalize_shared_features_true_method_reg_SIMPLE_WEIGHT_P_EX_0.20_opt_local_vs_reg}}} & \includegraphics[trim=10bp 30bp 40bp 30bp,clip,width=.45\linewidth]{{{Plots/ionosphere_new/results_28th__9h_10m_53s_cvex_true_cvrados_true_normalize_shared_features_true_method_reg_SIMPLE_WEIGHT_P_EX_0.20_prop_BH_loc_beaten_by_reg}}}\\ \hline\hline
\end{tabular}
\caption{Results on domain ionosphere, using the same convention as Table \ref{tab:exp_page}.\label{tab:exp_ionosphere}}
\end{table}

\begin{table}[t]
\centering
\begin{tabular}{ccc}\hline\hline
\rotatebox[origin=c]{90}{$p_s = 0.0$ \hspace{-4cm}} & \includegraphics[trim=10bp 30bp 40bp 30bp,clip,width=.45\linewidth]{{{Plots/winered/results_11th__13h_40m_15s_cvex_true_cvrados_true_normalize_shared_features_true_method_reg_SIMPLE_WEIGHT_P_EX_0.00_opt_local_vs_reg}}} & \includegraphics[trim=10bp 30bp 40bp 30bp,clip,width=.45\linewidth]{{{Plots/winered/results_11th__13h_40m_15s_cvex_true_cvrados_true_normalize_shared_features_true_method_reg_SIMPLE_WEIGHT_P_EX_0.00_prop_BH_loc_beaten_by_reg}}}\\
\rotatebox[origin=c]{90}{$p_s = 0.2$ \hspace{-4cm}} & \includegraphics[trim=10bp 30bp 40bp 30bp,clip,width=.45\linewidth]{{{Plots/winered/results_11th__13h_40m_15s_cvex_true_cvrados_true_normalize_shared_features_true_method_reg_SIMPLE_WEIGHT_P_EX_0.20_opt_local_vs_reg}}} & \includegraphics[trim=10bp 30bp 40bp 30bp,clip,width=.45\linewidth]{{{Plots/winered/results_11th__13h_40m_15s_cvex_true_cvrados_true_normalize_shared_features_true_method_reg_SIMPLE_WEIGHT_P_EX_0.20_prop_BH_loc_beaten_by_reg}}}\\ \hline\hline
\end{tabular}
\caption{Results on domain winered, using the same convention as Table \ref{tab:exp_page}.\label{tab:exp_winered}}
\end{table}

\begin{table}[t]
\centering
\begin{tabular}{ccc}\hline\hline
\rotatebox[origin=c]{90}{$p_s = 0.0$ \hspace{-4cm}} & \includegraphics[trim=10bp 30bp 40bp 30bp,clip,width=.45\linewidth]{{{Plots/winewhite/results_11th__14h_8m_39s_cvex_true_cvrados_true_normalize_shared_features_true_method_reg_SIMPLE_WEIGHT_P_EX_0.00_opt_local_vs_reg}}} & \includegraphics[trim=10bp 30bp 40bp 30bp,clip,width=.45\linewidth]{{{Plots/winewhite/results_11th__14h_8m_39s_cvex_true_cvrados_true_normalize_shared_features_true_method_reg_SIMPLE_WEIGHT_P_EX_0.00_prop_BH_loc_beaten_by_reg}}}\\
\rotatebox[origin=c]{90}{$p_s = 0.2$ \hspace{-4cm}} & \includegraphics[trim=10bp 30bp 40bp 30bp,clip,width=.45\linewidth]{{{Plots/winewhite/results_11th__14h_8m_39s_cvex_true_cvrados_true_normalize_shared_features_true_method_reg_SIMPLE_WEIGHT_P_EX_0.20_opt_local_vs_reg}}} & \includegraphics[trim=10bp 30bp 40bp 30bp,clip,width=.45\linewidth]{{{Plots/winewhite/results_11th__14h_8m_39s_cvex_true_cvrados_true_normalize_shared_features_true_method_reg_SIMPLE_WEIGHT_P_EX_0.20_prop_BH_loc_beaten_by_reg}}}\\ \hline\hline
\end{tabular}
\caption{Results on domain winewhite, using the same convention as Table \ref{tab:exp_page}.\label{tab:exp_winewhite}}
\end{table}

\begin{table}[t]
\centering
\begin{tabular}{ccc}\hline\hline
\rotatebox[origin=c]{90}{$p_s = 0.0$ \hspace{-4cm}} & \includegraphics[trim=10bp 30bp 40bp 30bp,clip,width=.45\linewidth]{{{Plots/phishing/results_11th__16h_47m_52s_cvex_true_cvrados_true_normalize_shared_features_true_method_reg_SIMPLE_WEIGHT_P_EX_0.00_opt_local_vs_reg}}} & \includegraphics[trim=10bp 30bp 40bp 30bp,clip,width=.45\linewidth]{{{Plots/phishing/results_11th__16h_47m_52s_cvex_true_cvrados_true_normalize_shared_features_true_method_reg_SIMPLE_WEIGHT_P_EX_0.00_prop_BH_loc_beaten_by_reg}}}\\
\rotatebox[origin=c]{90}{$p_s = 0.2$ \hspace{-4cm}} & \includegraphics[trim=10bp 30bp 40bp 30bp,clip,width=.45\linewidth]{{{Plots/phishing/results_11th__16h_47m_52s_cvex_true_cvrados_true_normalize_shared_features_true_method_reg_SIMPLE_WEIGHT_P_EX_0.20_opt_local_vs_reg}}} & \includegraphics[trim=10bp 30bp 40bp 30bp,clip,width=.45\linewidth]{{{Plots/phishing/results_11th__16h_47m_52s_cvex_true_cvrados_true_normalize_shared_features_true_method_reg_SIMPLE_WEIGHT_P_EX_0.20_prop_BH_loc_beaten_by_reg}}}\\ \hline\hline
\end{tabular}
\caption{Results on domain phishing, using the same convention as Table \ref{tab:exp_page}.\label{tab:exp_phishing}}
\end{table}

\begin{table}[t]
\centering
\begin{tabular}{ccc}\hline\hline
\rotatebox[origin=c]{90}{$p_s = 0.0$ \hspace{-4cm}} & \includegraphics[trim=10bp 30bp 40bp 30bp,clip,width=.45\linewidth]{{{Plots/wine/results_12th__14h_23m_8s_cvex_true_cvrados_true_normalize_shared_features_true_method_reg_SIMPLE_WEIGHT_P_EX_0.00_opt_local_vs_reg}}} & \includegraphics[trim=10bp 30bp 40bp 30bp,clip,width=.45\linewidth]{{{Plots/wine/results_12th__14h_23m_8s_cvex_true_cvrados_true_normalize_shared_features_true_method_reg_SIMPLE_WEIGHT_P_EX_0.00_prop_BH_loc_beaten_by_reg}}}\\
\rotatebox[origin=c]{90}{$p_s = 0.2$ \hspace{-4cm}} & \includegraphics[trim=10bp 30bp 40bp 30bp,clip,width=.45\linewidth]{{{Plots/wine/results_12th__14h_23m_8s_cvex_true_cvrados_true_normalize_shared_features_true_method_reg_SIMPLE_WEIGHT_P_EX_0.20_opt_local_vs_reg}}} & \includegraphics[trim=10bp 30bp 40bp 30bp,clip,width=.45\linewidth]{{{Plots/wine/results_12th__14h_23m_8s_cvex_true_cvrados_true_normalize_shared_features_true_method_reg_SIMPLE_WEIGHT_P_EX_0.20_prop_BH_loc_beaten_by_reg}}}\\ \hline\hline
\end{tabular}
\caption{Results on domain wine, using the same convention as Table \ref{tab:exp_page}.\label{tab:exp_wine}}
\end{table}

\begin{table}[t]
\centering
\begin{tabular}{ccc}\hline\hline
\rotatebox[origin=c]{90}{$p_s = 0.0$ \hspace{-4cm}} & \includegraphics[trim=10bp 30bp 40bp 30bp,clip,width=.45\linewidth]{{{Plots/statlog/results_12th__12h_46m_40s_cvex_true_cvrados_true_normalize_shared_features_true_method_reg_SIMPLE_WEIGHT_P_EX_0.00_opt_local_vs_reg}}} & \includegraphics[trim=10bp 30bp 40bp 30bp,clip,width=.45\linewidth]{{{Plots/statlog/results_12th__12h_46m_40s_cvex_true_cvrados_true_normalize_shared_features_true_method_reg_SIMPLE_WEIGHT_P_EX_0.00_prop_BH_loc_beaten_by_reg}}}\\
\rotatebox[origin=c]{90}{$p_s = 0.2$ \hspace{-4cm}} & \includegraphics[trim=10bp 30bp 40bp 30bp,clip,width=.45\linewidth]{{{Plots/statlog/results_12th__12h_46m_40s_cvex_true_cvrados_true_normalize_shared_features_true_method_reg_SIMPLE_WEIGHT_P_EX_0.20_opt_local_vs_reg}}} & \includegraphics[trim=10bp 30bp 40bp 30bp,clip,width=.45\linewidth]{{{Plots/statlog/results_12th__12h_46m_40s_cvex_true_cvrados_true_normalize_shared_features_true_method_reg_SIMPLE_WEIGHT_P_EX_0.20_prop_BH_loc_beaten_by_reg}}}\\ \hline\hline
\end{tabular}
\caption{Results on domain statlog, using the same convention as Table \ref{tab:exp_page}.\label{tab:exp_statlog}}
\end{table}

\begin{table}[t]
\centering
\begin{tabular}{ccc}\hline\hline
\rotatebox[origin=c]{90}{$p_s = 0.0$ \hspace{-4cm}} & \includegraphics[trim=10bp 30bp 40bp 30bp,clip,width=.45\linewidth]{{{Plots/steelplates/results_12th__15h_28m_20s_cvex_true_cvrados_true_normalize_shared_features_true_method_reg_SIMPLE_WEIGHT_P_EX_0.00_opt_local_vs_reg}}} & \includegraphics[trim=10bp 30bp 40bp 30bp,clip,width=.45\linewidth]{{{Plots/steelplates/results_12th__15h_28m_20s_cvex_true_cvrados_true_normalize_shared_features_true_method_reg_SIMPLE_WEIGHT_P_EX_0.00_prop_BH_loc_beaten_by_reg}}}\\
\rotatebox[origin=c]{90}{$p_s = 0.2$ \hspace{-4cm}} & \includegraphics[trim=10bp 30bp 40bp 30bp,clip,width=.45\linewidth]{{{Plots/steelplates/results_12th__15h_28m_20s_cvex_true_cvrados_true_normalize_shared_features_true_method_reg_SIMPLE_WEIGHT_P_EX_0.20_opt_local_vs_reg}}} & \includegraphics[trim=10bp 30bp 40bp 30bp,clip,width=.45\linewidth]{{{Plots/steelplates/results_12th__15h_28m_20s_cvex_true_cvrados_true_normalize_shared_features_true_method_reg_SIMPLE_WEIGHT_P_EX_0.20_prop_BH_loc_beaten_by_reg}}}\\ \hline\hline
\end{tabular}
\caption{Results on domain steelplates, using the same convention as Table \ref{tab:exp_page}.\label{tab:exp_steelplates}}
\end{table}

\begin{table}[t]
\centering
\begin{tabular}{cc}\hline\hline
\includegraphics[trim=5bp 25bp 40bp 10bp,clip,width=.45\linewidth]{{{Plots/ionosphere_new/results_27th__10h_36m_29s_cvex_true_cvrados_true_normalize_shared_features_true_method_reg_SIMPLE_WEIGHT_P_EX_0.20_p_function_of_bins_reg}}} & \includegraphics[trim=5bp 25bp 40bp 10bp,clip,width=.45\linewidth]{{{Plots/sonar/results_28th__14h_29m_53s_cvex_true_cvrados_true_normalize_shared_features_true_method_reg_SIMPLE_WEIGHT_P_EX_0.20_prop_glo_vs_reg_scatter}}}\\ \hline\hline
\end{tabular}
\caption{\textit{Left}: test error of \BBR~on domain ionosphere, as a function of the number of bins, aggregating all values of the number of peers $p$ and number of shared features $\dim({\mathcal{J}})$ used in Table \ref{tab:exp_ionosphere}; the green line denotes the average values. \textit{Right}: scatterplot of the test error of \BBR~($y$) vs that of the Oracle (learning using the complete entity-resolved domain). Points in the dark grey area (green) denote better performances of \BBR; points in the light grey area (blue) denote better performances of the Oracle (but not statistically better). Points in the white area (red) denote \textit{statistically} better performances of the Oracle (filled points:$p_s = 0.2$; empty points:$p_s = 0.8$). \label{tab:exp_is}}
\end{table}

\begin{table}[t]
\centering
\begin{tabular}{cc}\hline\hline
\includegraphics[trim=5bp 25bp 40bp 10bp,clip,width=.45\linewidth]{{{Plots/firmteacher2/results_24th__0h_17m_28s_cvex_true_cvrados_true_normalize_shared_features_true_method_reg_SIMPLE_WEIGHT_P_EX_0.20_opt_local_vs_dum}}} & \includegraphics[trim=5bp 25bp 40bp 10bp,clip,width=.45\linewidth]{{{Plots/mice/results_23th__20h_0m_8s_cvex_true_cvrados_true_normalize_shared_features_true_method_reg_SIMPLE_WEIGHT_P_EX_0.20_opt_local_vs_dum}}}\\ \hline\hline
\end{tabular}
\caption{Results of the dummy regularized \BBR~($\Gamma = \mathrm{Id}_d$) on domains firmteacher (left) and mice (right), following the convention of Table \ref{tab:exp_ionosphere} ($p_s = 0.2$).\label{tab:exp_fm}}
\end{table}

We remark that our framework is not formally comparable with \textsc{ER}, since the two address different problems. On one hand, \textsc{ER} has a much broader applicability than the problem object of this paper; learning on distributed datasets is less general than \textsc{ER}: in fact, we show a solution that bypasses \textsc{ER}. On the other hand, \emph{learning-based} \textsc{ER} \cite{bmAD} as well as manifold alignment techniques \cite{lkcDF} are viable only knowing some ground truth matches --- which are not required for working with rados. From another perspective, in concert with the \emph{open issues} in \cite{gmER}, we study ER as component of a pipeline for classification, and highlight how matching is not necessary for the purpose of learning.

In spite of those considerations, we can still draw comparisons with methods that learn on top of data merged through \textsc{ER} (Table \ref{table:2}). In both settings, no ID is shared between datasets but some attributes must be so, in order to allow entities comparison for matching or for building rados. Obviously, entity matching does not require the labels to be one of those shared attributes, while this is a fundamental hypothesis of our approach. Although, it is not as restrictive as may be expected at first: if just one peer has labels, then \textit{all} can obtain labels on their own data, via \emph{learning from label proportions} \cite{pnrcAN,qsclEL}: the label handling peer computes the label proportions per each block; the ``bags" are defined by examples matching a particular signature. Proportions are then shared among all other peers, which can train a classifier with them so as to obtain approximate labels for each observation.

To discuss time complexity, let us consider a simplified problem with only 2 peers with $m$ examples each in the (\textbf{VP}) scenario. In terms of complexity of fusion, if we assume that examples are uniformly distributed in the blocks, each block has size $m / m^\star$. \textsc{DRL} builds each block rado in time $O(m/H)$, with total cost linear in $m$. \textsc{ER} takes $O(m^2 / H^2 \cdot T_{sim})$ to match entities in each of the $H$ blocks, where $T_{sim}$ is the cost of evaluation a similarity function;  learning-based methods spend additional time for training; advanced blocking strategies can reduce the average complexity \cite{bkmAB,wmktgER, wgmJE}.

Most literature on distributed learning is concerned with limiting communication and designing optimal strategies for merging models \cite{bbfmDL, LIde}; beside that, previous works focus on horizontal split by observations, with few exceptions \cite{liDP}. In contrast, we exploit what is sufficient to merge \emph{about the data}. The communication protocol is extremely simple. Once rados are crafted locally, they are sent to a central learner in one shot. By Theorem \ref{thoptsol}, only $d$-dimensional $m^\star$ blocks rados are needed. \emph{Data is not accessed anymore} and learning takes place centrally. Moreover, rados help with data compression, being $m^\star \times d$, $m^\star \ll m$ the problem size. \textsc{ER} needs to transfer and learn from all entities, for a total size of $m \times d$.

Learning on data described by different feature sets is the topic of multiple view learning and co-training \cite{bmCL, snbAC}. To the best of our knowledge, co-training with unknown matches has not been addressed before. \cite{bswEC} presents a multi-view distributed algorithm with co-regularization; although it requires matches for all unlabelled examples.

In settings with multiple data providers, privacy can be crucial \cite{bbfmDL}. The agents have to trade off model enhancements and information leaks. A learner receives rados to train the model; this can be done by one of the agents, or by a third party --- paralleling multi-party ER scenarios \cite{cPP}. The only information sent through the channel consists of rados, while examples, with their individual sensible features, are never shared. Hardness results on reconstruct-ability of examples have been proven, along with \textsc{NP-hard} characterizations, and protection in the sense of differential privacy \cite{npfRO}. Furthermore, due to their compressive power, rados represent an alternative to bulk data collection \cite{sdklllpwBC}:  storing examples becomes superfluous. Regarding \textsc{ER}, since matching has the potential of de-anonimizing the entities, privacy is usually a very relevant issue to address \cite{cPP}. However, solutions are not straightforward, as proven by the vast amount of research on the topic \cite{vcvAT}; techniques based on partial share of attributes, anonymization or hashing can severely impair the process.

Even assuming labelled examples, no (observation, label) pair is actually available for training, and thus the task can be seen as weakly supervised \cite{ggwDO, pnncLF}. Although, a set of aggregate quantities, \emph{i.e.} sums of examples over subsets of the total sample (the rados), turns out to be enough for learning. Theorem \ref{equiv1} expresses a form of \emph{sufficiency} of the whole set of rados with regard to the square loss; a similar property is proposed for logistic loss in \cite{npfRO}. One of the $2^m$ rados the \textit{mean operator}, $\bm{\mu}_{{\mathcal{S}}} \defeq (1/m) \cdot \bm{\piup}_{\ve{y}}$, is formally proven a \textit{sufficient statistics} for the class for a wide set of losses \cite{pnrcAN, pnncLF}. This work, along with the cited predecessors, shows how the interplay between aggregate statistics and losses can lead to effective solutions to difficult learning problems.

\section{Conclusion}\label{sec:conc}

The key message of our paper is that Entity Matching addresses a very general \textit{but} difficult problem, and in the comparatively restricted context of supervised learning from distributed datasets, accurate learning evading the pitfalls of Entity Matching \textit{is} possible with Rademacher observations.
Rados have another advantage: they offer a cheap, easily parallelizable material which somehow ``compresses" examples while allowing accurate learning. They also offer readily available solution for guarantees private exchange of data in a distributed setting.
Finally, some domains display that there is significant room space for improvement of how cross-validation of optimized parameters are handled. This interesting problem comes in part from the fact that statistical properties of cross-validation on rados are \textit{not} the same as when carried out on examples; this particular aspect will deserve further analysis in the future.

\section{Acknowledgments}

NICTA is funded by the Australian Government through the Department of Communications and the Australian Research Council through the ICT Center of Excellence Program.

\bibliography{bib}
\bibliographystyle{plain}

\end{document}